\documentclass{article}
\usepackage{ijcai24}

\usepackage{times}
\usepackage{soul}
\usepackage{url}
\usepackage[hidelinks]{hyperref}
\usepackage[utf8]{inputenc}
\usepackage[small]{caption}
\usepackage{graphicx}
\usepackage{amsmath}
\usepackage{amsthm}
\usepackage{booktabs}
\usepackage{algorithm}
\usepackage{algorithmic}
\usepackage[switch]{lineno}
\usepackage{color}

\usepackage{bm}
\usepackage{amsfonts}
\usepackage{amssymb}

\newtheorem{theorem}{Theorem}
\newtheorem{definition}{Definition}

\title{Peptide Vaccine Design by Evolutionary Multi-Objective Optimization}

\author{
Dan-Xuan Liu
\and
Yi-Heng Xu\And
Chao Qian
\affiliations
National Key Laboratory for Novel Software Technology, Nanjing University, Nanjing 210023, China\\
School of Artificial Intelligence, Nanjing University, Nanjing 210023, China\\
\emails
liudx@lamda.nju.edu.cn,
201300025@smail.nju.edu.cn,
qianc@lamda.nju.edu.cn
}

\begin{document}

\maketitle

\begin{abstract}
Peptide vaccines are growing in significance for fighting diverse diseases. Machine learning has improved the identification of peptides that can trigger immune responses, and the main challenge of peptide vaccine design now lies in selecting an effective subset of peptides due to the allelic diversity among individuals. Previous works mainly formulated this task as a constrained optimization problem, aiming to maximize the expected number of peptide-Major Histocompatibility Complex (peptide-MHC) bindings across a broad range of populations by selecting a subset of diverse peptides with limited size; and employed a greedy algorithm, whose performance, however, may be limited due to the greedy nature. In this paper, we propose a new framework PVD-EMO based on Evolutionary Multi-objective Optimization, which reformulates Peptide Vaccine Design as a bi-objective optimization problem that maximizes the expected number of peptide-MHC bindings and minimizes the number of selected peptides simultaneously, and employs a Multi-Objective Evolutionary Algorithm (MOEA) to solve it. We also incorporate warm-start and repair strategies into MOEAs to improve efficiency and performance. We prove that the warm-start strategy ensures that PVD-EMO maintains the same worst-case approximation guarantee as the previous greedy algorithm, and meanwhile, the EMO framework can help avoid local optima. Experiments on a peptide vaccine design for COVID-19, caused by the SARS-CoV-2 virus, demonstrate the superiority of PVD-EMO.

\end{abstract}

\section{Introduction}
Peptide vaccines, composed of a set of peptides with the ability to selectively activate and proliferate T cells, are becoming increasingly essential in clinical treatment of a variety of diseases, including human immunodeficiency virus~\cite{HIV1,HIV2}, Alzheimer's disease~\cite{Alzheimer2,Alzheimer1}, and various forms of cancer~\cite{cancer1,cancer2}. Compared to traditional live-attenuated vaccines, peptide vaccine design offers a high degree of specificity and avoids the risks of infection caused by using entire pathogens.


Epitopes on peptides can bind to Major Histocompatibility Complex (MHC) molecules within the human body, thereby eliciting specific immune responses. Machine learning has greatly aided peptide (epitope) identification, involving detecting peptides that trigger immune responses~\cite{NetMHCpan1,NetMHCpan2}. Many impressive works have benefited from advanced peptide prediction methods~\cite{baruah2020immunoinformatics,bhattacharya2020development}. After predicting pathogen peptides, the main challenge in peptide vaccine design is selecting an effective peptide subset. This task is complicated by the extensive allelic diversity of human MHC molecules, resulting in significant variability in peptide-MHC combinations across individuals~\cite{robinson2020ipd,barker2023ipd}. Therefore, selecting a diverse subset of peptides that can be universally effective for a broad population poses a considerable challenge.

\subsection{Related Work}

Most previous work on peptide vaccine design typically relies on discrete optimization methods, such as integer linear programming and genetic algorithms, to maximize population coverage~\cite{oyarzun2015computer}. Liu \textit{et al.}~\shortcite{LiuD0G22} emphasized the importance of incorporating redundancy to boost the likelihood of effective immunogenic peptide presentation. They treated peptide vaccine design as a maximum $N$-times coverage problem, aiming to select peptides that ensure at least $N$ immunogenic peptide-MHC interactions per individual, and solved it with heuristic methods. However, this problem is NP-hard and, being non-submodular, does not allow polynomial-time constant factor approximation. This formulation also tends to underestimate the utility of a peptide until $N$-times coverage is achieved.

Dai \textit{et al.}~\shortcite{dai2023constrained} formulated peptide vaccine design as a constrained optimization problem, aiming to maximize the expected number of peptide-MHC bindings across a broad range of populations while choosing a limited set of diverse peptides. They introduced constraints on the size of the peptide subset and promoted dissimilar redundancies to avoid correlated failures, leading to the cardinality and pairwise constraints. The objective function, i.e., the expected number of peptide-MHC bindings across a wide population range, is proved to be both monotone and submodular, which enables their greedy algorithm, Optivax-P, to provide an approximation guarantee of $\max\{\mathrm{OPT}/(1+\Delta),\mathrm{OPT}_2/2\}$, where $\Delta> 0$, $\mathrm{OPT}$ denotes the optimal function value, and $\mathrm{OPT}_2$ is the optimal function value under additional constraints.

\subsection{Our Contribution}
Considering that the performance of the greedy algorithm may be limited due to its greedy nature, this paper proposes a new framework based on Evolutionary Multi-objective Optimization~\cite{knowles2001reducing,ecj15submodular,qian2015subset} for Peptide Vaccine Design, briefly called PVD-EMO, which reformulates peptide vaccine design as a bi-objective optimization problem that maximizes the expected number of peptide-MHC bindings and minimizes the number of selected peptides simultaneously, and employs a Multi-Objective Evolutionary Algorithm (MOEA) to solve it. PVD-EMO can be equipped with any MOEA to solve this bi-objective problem, and we employ the theoretically grounded GSEMO~\cite{Laumanns04} as well as the popular NSGA-II~\cite{nsgaii}. To boost efficiency, we also incorporate warm-start and repair strategies into MOEAs. The warm-start strategy seeds an initial population of PVD-EMO with a range of solution sizes, including the output solution of the previous greedy algorithm Optivax-P~\cite{dai2023constrained}. Meanwhile, the repair strategy is designed to steer PVD-EMO towards feasible solutions when it encounters infeasible regions, thereby improving its exploratory potential. We prove that the warm-start strategy ensures that PVD-EMO maintains the same worst-case approximation guarantee as the previous greedy algorithm Optivax-P. Additionally, by using an illustrative example of peptide vaccine design, we prove that Optivax-P will get trapped in local optima, while PVD-EMO can help avoid them. Experiments on a peptide vaccine design for COVID-19, caused by the SARS-CoV-2 virus, demonstrate the superiority of PVD-EMO over the state-of-the art algorithm Optivax-P.

\section{Peptide Vaccine Design}
Let $V=\{v_1,v_2,\ldots,v_{n}\}$ denote the set of peptides, and $M$ denote the set of MHC genotypes observed in the population. Peptide vaccines work by selecting an effective subset of peptides displayed on the cell surface of MHC proteins. The resulting peptide-MHC complexes activate the cellular immune system. Let $p_{v,m}$ be the probability that a peptide $v\in V$ is displayed by an individual's MHC genotype $m\in M$, with the assumption that these probabilities are independent across different peptides. The occurrence of peptide $v$ being displayed is termed a peptide-MHC hit, represented by the indicator function $\mathbb{I}(p_{v,m})$, which equals 1 iff the event occurs.

Different MHC alleles have different peptide binding properties, so it is important for a vaccine to trigger multiple peptide-MHC bindings, ensuring redundancy in the activation of T cell clonotypes in an individual. Given a subset $S$ of selected peptides, for an MHC genotype $m\in M$, the number of peptides in $S$ that are displayed can be represented as $\sum_{v\in S}\mathbb{I}(p_{v,m})$. Since too much redundancy may lead to unnecessary burden on the immune system, a threshold parameter $N\ge 0$ is used to limit the additional benefits until a person attains the desired $N$ hits, i.e., $\min\{\sum_{v\in S}\mathbb{I}(p_{v,m}), N\}$. Thus, the expected number of peptide-MHC bindings for a subset $S$ across the whole MHC genotypes set $M$ is 
\begin{equation}\label{eq-f}
    f(S)=\sum_{m\in M}w(m)\cdot\mathbb{E}[\min\{\sum_{v\in S}\mathbb{I}(p_{v,m}), N\}],
\end{equation}
where $w(m)$ denotes the weight corresponding to the percentage of the population with the genotype $m$. 

To obtain $f(S)$, we calculate $\mathbb{E}[\min\{\sum_{v\in S}\mathbb{I}(p_{v,m}), N\}]$ for each $m\in M$, and then sum up these expectations across all $m$. Assume that the indicator variables $\mathbb{I}(p_{v,m})$ are independent for each MHC genotype $m\in M$. For a given MHC genotype $m$, let $S_i\subseteq S$ denote the subset consisting of the first $i$ peptides of $S$ (assume an arbitrary order of the peptides in $S$), and $Y_i=\sum_{v\in S_i}\mathbb{I}(p_{v,m})$ denote the sum of $|S_i|$ independent Bernoulli trials. The distribution of $Y_0$ is trivially $P(Y_0=0)=1$, and we can iteratively compute the distribution of $Y_i$ as $P(Y_{i}=y)=P(Y_{i-1}=y-1)\cdot p+P(Y_{i-1}=y)\cdot(1-p)$, where $p$ denotes the probability $p_{v,m}$ that the new peptide $v\in S_i\setminus S_{i-1}$ is displayed by the MHC genotype $m$. Then, the distribution of $Z=\min\{\sum_{v\in S}\mathbb{I}(p_{v,m}),N\}$ satisfies that $\forall z<N: P(Z=z)= P(Y_{|S|}=z)$, and $P(Z=N)=P(Y_{|S|}\ge N)$, where $Y_{|S|}$ is just $\sum_{v\in S}\mathbb{I}(p_{v,m})$. The expectation is then computed using the resulting distribution.

The objective function, given by Eq.~(\ref{eq-f}), is monotone and submodular~\cite{dai2023constrained}. Let $\mathbb{R}$ denote the set of reals. A set function $f: 2^V \rightarrow \mathbb{R}$ is monotone if $\forall X \subseteq Y \subseteq V$, $f(X) \leq f(Y)$. As more peptides will not worsen the value, the monotonicity is satisfied naturally. A set function $f$ is submodular~\cite{nemhauser1978analysis} if it satisfies the ``diminishing returns'' property, i.e., $\forall X\subseteq Y\subseteq V, v\notin Y$, \begin{align*}f(X\cup \{v\})-f(X)\ge f(Y\cup\{v\})-f(Y),\end{align*} 
which implies that the benefit of adding a peptide to a set will not increase as the set extends.

The peptide vaccine design problem is subject to two types of constraints. The first is a cardinality constraint that ensures the selected peptide subset $S$ not exceed a given size $k$, i.e., $|S|\leq k$, which is crucial for increasing the vaccine's stability, as well as reducing production costs. The second type involves pairwise constraints, which prevent any two peptides in the subset $S$ from being similar. Maintaining dissimilarity among peptides is crucial for vaccine design, as similar peptides might fail for the same reason, reducing the vaccine's overall efficacy. Let $G=(V,E)$ denote a graph, where the vertices correspond to the peptide set $V$, and the edges connect peptides that are deemed similar. A solution that satisfies the set of pairwise constraints corresponds to an independent set in the graph $G=(V,E)$; that is, for any two peptides $v_i$, $v_j$ in the subset $S\subseteq V$, there is no edge $(v_i,v_j)$ in $E$.

In~\cite{dai2023constrained}, peptide vaccine design has been formulated as the problem of maximizing the objective function $f(S)$, which characterizes the expected number of peptide-MHC bindings across a broad range of populations, subject to a cardinality constraint and a set of pairwise constraints, as presented in Definition~\ref{def-problem}.

\begin{definition}[Peptide Vaccine Design]\label{def-problem}
Given a set of peptides $V=\{v_1,v_2,\cdots,v_n\}$, a set of MHC genotypes $M$, probabilities $p_{v,m}$ of binding between each peptide $v$ and each MHC genotype $m$, a weight function $w(m)$ corresponding to the percentage of the population with each MHC genotype $m$, a threshold $N\ge 0$, a budget $k$, and a similarity graph $G=(V,E)$, the goal of peptide vaccine design is to find a subset $S\subseteq V$ of peptides that maximizes $f(S)$ subject to a size constraint and a set of pairwise constraints. That is:
\begin{align}\label{eq-problem}
&\mathop{\arg\max}\limits_{S\subseteq V} f(S)=\sum_{m\in M}w(m)\cdot\mathbb{E}[\min\{\sum\limits_{v\in S}\mathbb{I}(p_{v,m}),N\}]\nonumber\\ 
&\qquad \quad \text{s.t.}\quad |S|\leq k \enspace  \&\enspace   \forall v_i,v_j\in S, (v_i,v_j)\notin E.
\end{align}
\end{definition}

\subsection{Previous Algorithm}
For the peptide vaccine design problem in Definition~\ref{def-problem}, Dai and Gifford~\shortcite{dai2023constrained} introduced a greedy algorithm, Optivax-P, which starts from an empty set and iteratively adds peptides with the highest marginal gain on $f$, satisfying given constraints. Let $\Delta$ denote the maximum degree of the similarity graph $G$. When $\Delta=0$, implying no pairwise constraints, the problem simplifies to optimizing a monotone submodular function with a cardinality constraint, where Optivax-P achieves the optimal $(1-1/e)$ polynomial-time approximation ratio~\cite{nemhauser1978analysis}. For $\Delta>0$, indicating pairwise constraints, Optivax-P can find a solution $\hat{S}$ satisfying $f(\hat{S}) \geq \max\{\mathrm{OPT}/(1+\Delta),\mathrm{OPT}_2/2\}$, where $\mathrm{OPT}$ and $\mathrm{OPT}_2$ denote the optimal function values for the original and a more constrained problem, respectively. For the more constrained problem, $G$ is replaced by $G^2$, linking vertices within two steps in $G$.


\section{PVD-EMO Framework}
Inspired by the excellent performance of MOEAs for solving general subset selection problems~\cite{ecj15submodular,qian2015subset,QianSYT17,QianS0TZ17,QianZT018,QianYTYZ19,0001BF20,Qian20,0001Q0021,qian2021multiobjective,0002Z022,QianLZ22,RoostapourNNF22,QianLFT23,zhang2023sparsity}, we propose a new Peptide Vaccine Design framework based on Evolutionary Multi-objective Optimization, called PVD-EMO. A subset $S$ of $V$ can be naturally represented by a Boolean vector $\bm{s} \in\{0,1\}^n$, where the $i$-th bit $s_{i}=1$ iff the $i$-th peptide in $V$ is contained by $S$. We will not distinguish $\bm{s}\in\{0,1\}^n$ and its corresponding subset $\{v_{i}\in V \mid s_{i}=1\}$ for notational convenience.

As presented in Algorithm~\ref{alg:PVD-EMO}, PVD-EMO first reformulates the original peptide vaccine design problem in Definition~\ref{def-problem} as a bi-objective maximization problem 
\begin{align}\label{eq-bi-objective}
\arg\max\nolimits_{\bm{s} \in \{0,1\}^n} (f_1(\bm{s}),f_2(\bm{s})),
\end{align}
\begin{align*}
\text{where\qquad} f_1(\bm{s}) = \begin{cases}
	f(\bm{s}), &{\bm{s}} \text{ is feasible}\\
	-1, &{\text{otherwise}}
\end{cases},\quad
f_{2}(\bm{s})=-|\bm{s}|.
\end{align*}
That is, the first objective $f_1$ equals the original objective $f$ (i.e., the expected number of peptide-MHC bindings across a broad range of populations) for feasible solutions satisfying the cardinality and pairwise constraints, while $-1$ for infeasible ones; the second objective $f_{2}(\bm{s})=-|\bm{s}|=-\sum_{i=1}^{n}s_i$ is the opposite of the subset size. The domination relationship in Definition~\ref{def_Domination} is used to compare solutions. A solution is Pareto optimal if no other solution dominates it. The collection of objective vectors of all Pareto optimal solutions is called the Pareto front.

\begin{definition}[Domination]\label{def_Domination}
For two solutions $\bm s$ and $\bm{s}'$,\\
1. $\bm{s}$ \emph{weakly dominates} $\bm{s}'$ (i.e., $\bm{s}$ is \emph{better} than $\bm{s}'$, denoted by $\bm{s} \succeq \bm{s}'$) if \;$\forall i: f_i(\bm{s}) \geq f_i(\bm{s}')$;\\
2. ${\bm{s}}$ \emph{dominates} $\bm{s}'$ (i.e., $\bm{s}$ is \emph{strictly better} than $\bm{s}'$, denoted by $\bm{s} \succ \bm{s}'$) if ${\bm{s}} \succeq \bm{s}' \wedge \exists i: f_i(\bm{s}) > f_i(\bm{s}')$;\\
3. $\bm{s}$ and $\bm{s}'$ are \emph{incomparable} if neither $\bm{s} \succeq \bm{s}'$ nor $\bm{s}' \succeq \bm{s}$.
\end{definition}

\begin{algorithm}[t!]\caption{PVD-EMO Framework}\label{alg:PVD-EMO}
\textbf{Input}: a peptide vaccine design problem instance, a budget $k$ and a similarity graph $G=(V,E)$\\
\textbf{Output}: a subset of $V$ with size at most $k$\\
\textbf{Process}:
    \begin{algorithmic}[1]
    \STATE Construct two objective functions $f_1(\bm{s})$ and $f_2(\bm{s})$ to be maximized, as presented in Eq.~(\ref{eq-bi-objective});
    \STATE Apply an MOEA to solve the bi-objective problem;
    \STATE \textbf{return} the best feasible solution in the final population generated by the MOEA
    \end{algorithmic}
\end{algorithm}

After constructing the bi-objective problem in Eq.~(\ref{eq-bi-objective}), PVD-EMO employs an MOEA to solve it, as shown in line~2 of Algorithm~\ref{alg:PVD-EMO}. Evolutionary algorithms (EAs), inspired by Darwin’s theory of evolution, are general-purpose randomized heuristic optimization algorithms~\cite{back:96}, mimicking variational reproduction and natural selection. Starting from an initial population of solutions, EAs iteratively reproduce offspring solutions by crossover and mutation, and select better ones from the parent and offspring solutions to form the next population. The population-based search of EAs matches the requirement of multi-objective optimization, i.e., EAs can generate a set of Pareto optimal solutions by running only once. Thus, EAs have become the most popular tool for multi-objective optimization~\cite{coello2007evolutionary,hong2021evolutionary,yang2024reducing,liang2024evolutionary}, and the corresponding algorithms are also called MOEAs. After running a number of iterations, the best feasible solution (i.e., the solution having the largest $f$ value while satisfying the constraints) will be selected from the final population as the output, as shown in line~3 of Algorithm~\ref{alg:PVD-EMO}. Note that the aim of PVD-EMO is to find a good solution of the original peptide vaccine design problem in Definition~\ref{def-problem}, rather than the Pareto front of the reformulated bi-objective problem in Eq.~(\ref{eq-bi-objective}). That is, the bi-objective reformulation is an intermediate process. The introduction of the second objective $f_2$ can naturally bring a diverse population, which may lead to better optimization performance. 

PVD-EMO can be equipped with any MOEA, and we employ the theoretically grounded GSEMO~\cite{Laumanns04} as well as the popular NSGA-II~\cite{nsgaii}. To boost efficiency, we incorporate warm-start and repair strategies into MOEAs, and also design a strategy to accelerate the objective evaluation. Next, we will introduce them in detail.

\subsection{PVD-GSEMO-WR Algorithm}

PVD-EMO which employs the GSEMO algorithm and uses both Warm-start and Repair strategies is called PVD-GSEMO-WR, as presented in Algorithm~\ref{alg:PVD-GSEMO}. PVD-GSEMO-WR first constructs the bi-objective problem in Eq.~(\ref{eq-bi-objective}) in line~1 of Algorithm~\ref{alg:PVD-GSEMO}. After that, it starts with an initial population created by the warm-start strategy in line~2, and iteratively improves the quality of solutions in the population $P$ (lines~3–10). The warm-start strategy presented in Algorithm~\ref{alg-warm-start} randomly generates a feasible solution for each size in $\{0, 1, \ldots, k-1\}$, while for the solution of size $k$, it uses the output from the previous greedy algorithm, Optivax-P~\cite{dai2023constrained}. Among these solutions, the non-dominated ones will be included in the initial population $P$. In each iteration, a parent solution $\bm{s}$ is selected from the current population $P$ uniformly at random (line~4), and used to generate an offspring solution $\bm{s}'$ by bit-wise mutation (line~5), which flips each bit of $\bm{s}$ independently with probability $1/n$. Then the offspring solution $\bm{s}'$ will go through a repair strategy (line~6), as presented in Algorithm~\ref{alg-repair}, which can fix an offspring solution $\bm{s}'$ such that it no longer violates the pairwise constraints, and will be introduced later. The repaired offspring solution $\bm{s}''$ equals to $\bm{s}'$ if $\bm{s}'$ has already been feasible. Then, $\bm{s}''$ is used to update the population $P$ (lines 7–9). If $\bm{s}''$ is not dominated by any solution in $P$ (line~7), it will be added into $P$, and meanwhile, those solutions weakly dominated by $\bm{s}''$ will be deleted (line~8). This updating procedure makes the population $P$ always contain incomparable solutions. Furthermore, $P$ always contains feasible solutions, because 1) the repair strategy can fix the violation on the pairwise constraints; 2) a solution violating the size constraint has bad values on both objectives according to Eq.~(\ref{eq-bi-objective}), and will not be included into the population.

\begin{algorithm}[t!]\caption{PVD-GSEMO-WR Algorithm}\label{alg:PVD-GSEMO}
\textbf{Input}: a peptide vaccine design problem instance, a budget $k$ and a similarity graph $G=(V,E)$\\
\textbf{Output}: a subset of $V$ with size at most $k$\\
\textbf{Process}:
    \begin{algorithmic}[1]
     \STATE Construct two objective functions $f_1(\bm{s})$ and $f_2(\bm{s})$ to be maximized, as presented in Eq.~(\ref{eq-bi-objective});
    \STATE Initialize the population $P$ by Warm-Start Strategy;
    \REPEAT
    \STATE  Choose $\bm s$ from $P$ uniformly at random;
    \STATE  Create $\bm{s}'$ by flipping each bit of $\bm s$ with prob. $1/n$;
    \STATE  $\bm{s}''\gets$ Repair Strategy~($\bm{s}$, $\bm{s}'$, $G$);
    \IF{$\nexists \bm z \in P$ such that $\bm z \succ \bm {s}''$} \STATE $P \gets (P \setminus \{\bm z \in P \mid \bm {s}'' \succeq \bm z\}) \cup \{\bm {s}''\}$
    \ENDIF
    \UNTIL{some criterion is met} 
    \RETURN the best feasible solution in $P$
    \end{algorithmic}
\end{algorithm}

In Algorithm~\ref{alg-repair}, the repair strategy reviews each bit of parent solution $\bm{p}$ and offspring solution $\bm{o}$ (lines~1-7). When a bit $i$ flips from 0 in $\bm{p}$ to 1 in $\bm{o}$ (line~2), it identifies a set $Q$ of indices in $\bm{o}$ that are connected to peptide $v_i$ in edge set $E$ (line~3). The algorithm then randomly keeps one index $q$ from $Q$ unchanged (line~4) and sets the other connected bits in $Q \setminus {q}$ to 0 (line~5). This ensures that no pairwise constraints are violated in the first $i$ bits of $\bm{o}$. After processing all bits, the repaired $\bm{o}$, free from pairwise constraint violations, is returned (line~8).


\begin{algorithm}[t!]
\caption{Warm-Start Strategy}\label{alg-warm-start}
\textbf{Input}: solution $\bm{s}_g$ output by the greedy algorithm Optivax-P\\
\textbf{Process}:
\begin{algorithmic}[1]
\STATE $P=\{\bm{s}_g\}$;
\FOR{$i=0$ to $k-1$}
\STATE Randomly create a feasible solution $\bm{s}$ of size $i$;
\IF{$\nexists \bm z \in P$ such that $\bm z \succ \bm {s}$} \STATE $P \gets (P \setminus \{\bm z \in P \mid \bm {s} \succeq \bm z\}) \cup \{\bm {s}\}$
\ENDIF
\ENDFOR
\RETURN an initial population $P$
\end{algorithmic}
\end{algorithm}

\begin{algorithm}[t!]
\caption{Repair Strategy}\label{alg-repair}
\textbf{Input}: a similarity graph $G=(V,E)$, a parent solution $\bm{p}$ and an offspring solution $\bm{o}$\\
\textbf{Process}:
\begin{algorithmic}[1]
\FOR{$i=1$ to $n$}
\IF{$o_{i}=1$ and $p_{i}=0$}
\STATE $Q=\{j~|~(v_{i},v_{j})\in E ~\&~ o_{j}=1\}\cup \{i\}$;
\STATE Choose $q$ from $Q$ uniformly at random;
\STATE Set $o_p$ to $0$, for any $p \in Q\setminus \{q\}$
\ENDIF
\ENDFOR
\RETURN the offspring solution $\bm{o}$
\end{algorithmic}
\end{algorithm}

\subsection{PVD-NSGA-II-WR Algorithm}

NSGA-II is a popular MOEA which incorporates two substantial features, i.e., non-dominated sorting and crowding distance. For a detailed description of NSGA-II, please refer to~\cite{nsgaii}. PVD-EMO which employs NSGA-II for multi-objective optimization and incorporates both Warm-start and Repair strategies is called PVD-NSGA-II-WR.

In Section~\ref{theory}, we will prove that PVD-GSEMO-WR and PVD-NSGA-II-WR achieve the same theoretical guarantee as Optivax-P~\cite{dai2023constrained} while better avoiding local optima. In Section~\ref{experiments}, we will show their superior performance over Optivax-P in real-world experiments.


\subsection{Acceleration of Objective Evaluation}

In each iteration of PVD-EMO, we need to evaluate the objective value of a newly generated solution $\bm{s}''$, i.e.,  $f(S'')=\sum_{m\in M}w(m)\cdot\mathbb{E}[\min\{\sum_{v\in S''}\mathbb{I}(p_{v,m}), N\}]$, where $S''$ is the corresponding subset of $\bm{s}''$. For each $m\in M$, a random variable $Y=\sum_{v\in S''}\mathbb{I}(p_{v,m})$ is the sum of $|S''|\leq k$ independent Bernoulli trials $\mathbb{I}(p_{v,m})$, and we can get the probability distribution of $Y$ by iterated convolutions, which costs $O(k^2)$ time. $\mathbb{E}[\min\{\sum_{v\in S''}\mathbb{I}(p_{v,m}), N\}]$ then can be calculated by using $\mathbb{E}[h(y)]=\sum_{y}h(y) \cdot P(Y=y)$, where $h(y)=\min\{y,N\}$, which requires $O(k^3)$ time. Thus, the total time of computing $f(S'')$ is $O(|M| k^3)$, which is expensive.

Next, we give an acceleration strategy of evaluating $f(S'')$. For each $m\in M$, let $D^{X}$ denote the probability distribution of a random variable $\sum_{v\in X}\mathbb{I}(p_{v,m})$, and $D_j^{X}$ denote the probability of attaining $j$ hits from a peptide subset $X$, i.e., $D_j^{X}=P(D^{X}=j)$. The probability distribution $D^{\emptyset}$ of an empty subset satisfies that $D^{\emptyset}_0\!=\!1$ and $D^{\emptyset}_j\!=\!0$ for all $0<j\leq k$. To accelerate the evaluation process, we build a recursive relation between the distributions $D^{S''}$ and $D^{S}$, where $S$ is a parent solution of $S''$. Specifically, we can compute $D^{S''}$ based on $D^{S}$ by recursively adding peptides from $X^{+}=\{v_i~|~v_i\notin S \wedge v_i\in S''\}$ and then recursively deleting peptides from $X^{-}=\{v_i~|~v_i\in S \wedge v_i\notin S''\}$. The distribution $D^{X\cup\{v_i\}}$ is updated based on $D^{X}$:
\begin{align*}
D_j^{X\cup\{v_i\}}= \begin{cases}
	(1-p_{v_i,m}) D_j^{X},\qquad\qquad\qquad\qquad j=0, \\
	(1-p_{v_i,m}) D_j^{X}\!+\!p_{v_i,m} \cdot D_{j-1}^{X},~0<\!j\!\leq \!|X|,\\
        p_{v_i,m} \cdot D_{j-1}^{X}, \qquad\qquad\qquad\quad j=|X|+1,\\
        \qquad 0, \qquad\qquad\qquad\quad\quad |X|+1<j\leq k.\\
\end{cases}
\end{align*}
Similarly, the probability distribution $D^{X\setminus \{v_i\}}$ is updated by backtracking from $D^{X}$:
\begin{align*}
D_j^{X\setminus\{v_i\}}= \begin{cases}
	D_j^{X}/(1-p_{v_i,m}) ,\qquad\quad\quad\quad j=0, \\
	\dfrac{D_j^{X}-p_{v_i,m} \cdot D_{j-1}^{X\setminus\{v_i\}}}{1-p_{v_i,m}} , \quad\quad 0<j< |X|,\\
        \qquad 0, \qquad\qquad\qquad\quad\quad\quad |X|\leq j\leq k.\\
\end{cases}
\end{align*}
Note that the commonly used bit-wise mutation operator in line~5 of Algorithm~\ref{alg:PVD-GSEMO} will flip one bit in expectation, that is, it will add or delete one peptide in expectation. Consequently, the expected time complexity for calculating the probability distribution $D^{S''}$ of an offspring solution $S''$ is $O(k)$, which is faster than the $O(k^2)$ time required by direct computation. We store the utility function $h(y) = \min\{y, N\}$ in a data structure that supports random access, such as an array, which allows us to compute its expectation in $O(k)$ time. By implementing these optimizations, the overall time required to compute the objective function can be reduced from $O(|M|k^3)$ to $O(|M|k)$.


\section{Theoretical Analysis}\label{theory}

For peptide vaccine design in Definition~\ref{def-problem}, the objective function $f(\bm s)$ has been proved to be monotone and  submodular~\cite{dai2023constrained}. The greedy algorithm Optivax-P can provide an approximation guarantee of $\max\{\mathrm{OPT}/(1+\Delta),\mathrm{OPT}_2/2\}$, where $\Delta> 0$, $\mathrm{OPT}$ denotes the optimal function value of Eq.~(\ref{eq-problem}), and $\mathrm{OPT}_2$ is the optimal function value under additional constraints. By maximizing $f_1(\bm{s})=f(\bm s)$ and $f_2(\bm{s})=-|\bm{s}|$ simultaneously, PVD-GSEMO-WR or PVD-NSGA-II-WR can achieve the same theoretical guarantee as Optivax-P, as shown in Theorem~\ref{theo-gurantee}. 

\begin{theorem}\label{theo-gurantee}
For peptide vaccine design in Definition~\ref{def-problem}, PVD-GSEMO-WR, or PVD-NSGA-II-WR with a population size of at least $4(k+1)$, can achieve the same approximation guarantee as the previous greedy algorithm Optivax-P.
\end{theorem}

The proof is straightforward, because PVD-GSEMO-WR and PVD-NSGA-II-WR begin with an initial population containing the greedy solution output by Optivax-P, due to the warm-start process. The population update mechanism of PVD-GSEMO-WR ensures that if a solution is deleted, there must be another solution in the population that weakly dominates the deleted one. According to the recent theoretical work~\cite{zheng2022first}, if the population size is at least four times of the size of the Pareto front (which is no larger than $k+1$ here), PVD-NSGA-II-WR will keep at least one solution for each objective vector in the first non-dominated front of the population. Consequently, PVD-GSEMO-WR and PVD-NSGA-II-WR are guaranteed to obtain a solution that is at least as good as the one obtained by Optivax-P.



Using a peptide vaccine design example from Definition~\ref{def-example}, Theorem~\ref{theo-example} shows that Optivax-P falls into local optima, while PVD-GSEMO-WR can avoid them and achieve the global optimum, which excludes the peptide $v_1$. The proof shows that due to its greedy nature, Optivax-P first selects $v_1$ and is misled. In contrast, PVD-GSEMO-WR either avoids $v_1$ via bit-wise mutation or escapes local optima using a repair strategy, finally following the population's path to the global optimum.

\begin{definition}[An Example of Peptide Vaccine Design]\label{def-example}
Let $V=\{v_1,v_2,\cdots,v_n\}$ and  $2\leq k\leq n$. The similarity graph $G(V,E)$ is shown in Figure~\ref{fig-example} where the two nodes connected by each edge are considered similar. Each node in the graph has a degree of at most 1, i.e., it can form only one edge, except for $v_1$ which has a degree of 2. The objective function $f$ has properties as follows: 
\begin{itemize}
    \item  For any feasible solution $X\subseteq V$ with $|X|\leq k-1$, and $\forall v_x\in V\setminus X $ such that $\forall v_y\in X, (v_x,v_y)\notin E$, it holds
    \begin{equation}\label{eq-property-1}
        f(X\cup \{v_x\})=f(X)+f(\{v_x\});
    \end{equation}
    \item For any $v_x,v_y\in V\setminus\{v_1,v_2,v_3\}$, $v \in \{v_2,v_3\}$, it holds 
    \begin{align}\label{eq-property-2}
        f(\{v_x\})=f(\{v_y\})<f(\{v\}) <f(\{v_1\});
    \end{align}
    \item For any $v_x\in V\setminus\{v_1\}$ such that $(v_1,v_x)\notin E$, it holds 
    \begin{equation}\label{eq-property-3}
        f(\{v_1,v_x\})<f(\{v_2,v_3\}).
    \end{equation}
\end{itemize}
\end{definition}

\begin{figure}
    \centering
    \includegraphics[width=0.75\linewidth]{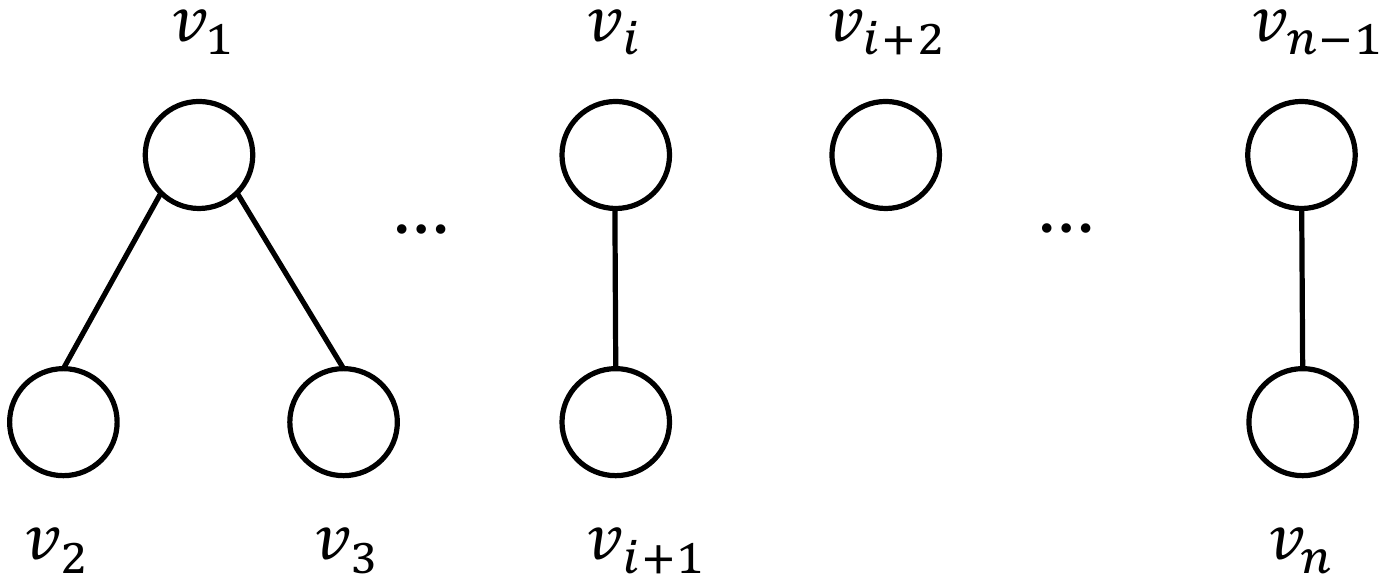}
    \caption{The similarity graph $G=(V,E)$ of an example of peptide vaccine design, where the vertices correspond to the peptides, and edges exist between the peptides that are deemed similar.}
    \label{fig-example}
\end{figure}

\begin{theorem}\label{theo-example}
For the peptide vaccine design example given in Definition~\ref{def-example}, PVD-GSEMO-WR can find an optimal solution within an expected number of iterations $O(kn^2)$, while the greedy algorithm Optivax-P cannot.
\end{theorem}
\begin{proof}
We first analyze the property of the optimal solutions, and then show that Optivax-P cannot find an optimal solution. Finally, we derive the expected number of iterations required for PVD-GSEMO-WR to find an optimal solution.

Let $X^i\subseteq V$ denote a feasible solution with $|X^i|=i$ and the set $\{v_2,v_3\}\nsubseteq X^i$. For any $X^2$ and $v_x\in V$ such that $(v_1,v_x)\notin E$, it holds that 
\begin{align}\label{eq-6}
    f(X^2)&=\sum\nolimits_{v\in X^2}f(\{v\})\leq f(\{v_1\})+f(\{v_x\})\nonumber\\
    &=f(\{v_1,v_x\})<f(\{v_2,v_3\}),
\end{align}
where the equalities hold by applying Eq.~(\ref{eq-property-1}), the first inequality holds by Eq.~(\ref{eq-property-2}), that is, $\forall v\in V, f(\{v\})\leq f(\{v_1\})$, and the last inequality is by Eq.~(\ref{eq-property-3}). This implies that the solution $\{v_2,v_3\}$ is the unique optimal solution for $k=2$. When $2<k<n$, we can similarly derive that for any $X^k=X^2\cup X^{k-2}$, it holds that
\begin{equation}\label{eq-8}
    f(X^k)\!=\!f(X^2)\!+\!f(X^{k-2})<f(\{v_2,v_3\})\!+\!f(X^{k-2}),
\end{equation}
where the equality holds by Eq.~(\ref{eq-property-1}), and the inequality holds by Eq.~(\ref{eq-6}). Here $X^{k-2}$ is assumed to not contain the peptides $v_1$, $v_2$ and $v_3$. If $v_1\in X^{k-2}$, we can exchange $v_1$ with a peptide $v$ from $X^2$ to ensure $v_1\notin X^{k-2}$. The same exchange process can be applied if $v_2\in X^{k-2}$ or $v_3\in X^{k-2}$. Note that by the definition of $X^i$, we know that $X^k$ can contain at most one of the three peptides $v_1$, $v_2$ and $v_3$; otherwise, $X^k$ will violate the pairwise constraints in Figure~\ref{fig-example} or the condition $\{v_2,v_3\}\nsubseteq X^i$. Eq.~(\ref{eq-property-2}) states that for any $v_x,v_y\in V\setminus\{v_1,v_2,v_3\}$, $f(\{v_x\})=f(\{v_y\})$; thus, $f(X^{k-2})$ has a constant value, denoted as $C\ge 0$. Combining Eq.~(\ref{eq-8}), we obtain that $f(X^k)< f(\{v_2,v_3\})+C$. Thus, we can find that a feasible solution $O$ is optimal iff $O=\{v_2,v_3\}\cup X^{k-2}$. 

The greedy algorithm Optivax-P first selects $v_1$, which has the largest marginal gain. Next, Optivax-P deletes all peptides from the set $V$ that are connected to $v_1$, including $v_2$ and $v_3$. Thus, the optimal solution cannot be found.

For the PVD-GSEMO-WR algorithm, the problem is implemented as maximizing $f(X)$ and minimizing $|X|$ simultaneously. We then prove that PVD-GSEMO-WR can find an optimal solution $O=\{v_2,v_3\}\cup X^{k-2}$ benefiting from the bit-wise mutation operator and repair strategy, respectively. Let $O^i~(2\leq i\leq k)$ denote the best feasible solution with $|O^i|=i$. We can verify that $O^i$ must contain $v_2$ and $v_3$, and inserting a specific element into $O^i$ can generate $O^{i+1}$. We reiterate that the definition of $X^i\subseteq V$ represents a feasible solution with $|X^i|=i~(0\leq i\leq k)$ and the set $\{v_2,v_3\}\nsubseteq X^i$. In the following proof, we first show that how PVD-GSEMO-WR can find $O^j$ from $X^i$ ($i<j\leq k$), and then follow the path $O^{j}\rightarrow O^{j+1}\rightarrow \cdots \rightarrow O^{k}=O$ to produce an optimal solution.

\textbf{[Bit-wise mutation]} When $X^i$ does not contain the peptide $v_1$, the idea is that flipping at most two 0-bits (corresponding to $v_2$ and $v_3$) of the solution $X^i$ and keeping other bits unchanged can find $O^{i+1}$ or $O^{i+2}$. Note that $X^i$ may contain $v_2$ or $v_3$, but will not contain them simultaneously due to the condition $\{v_2,v_3\}\nsubseteq X^i$. As the solutions in the population $P$ are incomparable and $f_2(X)=-|X|$, $P$ contains at most one solution for each subset size $0,1,\cdots,k$. Furthermore, the empty solution $X^0$ is guaranteed to be present in the initial population. Thus, the probability of generating $O^{i+1}$ or $O^{i+2}$ from $X^i$ is at least $(1/|P|)\cdot(1/n)^2\cdot(1-1/n)^{n-2}\ge 1/(en^2|P|)$, where $1/|P|$ is the probability of selecting $X^i$ by uniform selection in line~4 of Algorithm~\ref{alg:PVD-GSEMO}, and $(1/n)^2\cdot(1-1/n)^{n-2}$ is the probability of flipping two specific 0-bits (corresponding to $v_2$ and $v_3$) of the solution $X^i$ while keeping other bits unchanged in line~5. Because $|P| \leq k+1$, the expected number of iterations for finding $O^{i+1}$ or $O^{i+2}$ is at most $en^2|P|\leq e(k+1)n^2=O(kn^2)$.

\textbf{[Repair strategy]} We now show how PVD-GSEMO-WR escapes a local optimal solution $X^i~(1\leq i<k)$ that includes $v_1$. This is achieved by employing the repair strategy. The process begins by simultaneously flipping the two 0-bits in $X^i$ corresponding to $v_2$ and $v_3$. This action generates an infeasible solution ${v_1}\cup O^{i+1}$, which now contains the edges $(v_1,v_2)$ and $(v_1,v_3)$ in Figure~\ref{fig-example}. Subsequently, the repair strategy, as specified in lines~3--4 of Algorithm~\ref{alg-repair}, randomly chooses to remove either $v_1$ or $v_2$. Assuming it opts to exclude $v_1$, we then obtain the solution $O^{i+1}$. Thus, the probability of generating $O^{i+1}$ from $X^i$ is at least $(1/|P|)\cdot(1/n)^2\cdot(1-1/n)^{n-2}\cdot (1/2) \ge 1/(2en^2|P|)$, where $1/2$ is the probability of excluding $v_1$ from $\{v_1, v_2\}$ in line~5 of Algorithm~\ref{alg-repair}. This implies that the expected number of iterations is at most $2en^2|P|=O(kn^2)$.

\textbf{Taking the maximum} of the expected number of iterations, $O^j$ can be generated from $X^i$ ($i<j$) within $O(kn^2)$ expected number of iterations. Note that once $O^{j}~(2\leq j\leq k)$ is generated, it will always be kept in $P$, since it cannot be dominated by any other solution. The probability of the event $O^{j} \rightarrow O^{j+1}$ is at least $(1/|P|)\cdot(1/n)\cdot(1-1/n)^{n-1}\ge 1/(en|P|)$, since it is sufficient to select $O^{j}$ in line~4 of Algorithm~\ref{alg:PVD-GSEMO}, and then flip only one specific 0-bit corresponding to a peptide that has no edges with any peptide belonging to the subset $O^{j}$ in Figure~\ref{fig-example}. Because the length of the path $O^{j}\rightarrow O^{j+1}\rightarrow \cdots \rightarrow O^{k}=O$ is at most $k-2$, the total expected number of iterations for finding an optimal solution $O$ is at most $O(kn^2)+(k-2)\cdot en|P|=O(kn^2)$, implying that the theorem holds.
\end{proof}

From the proofs, we can find that the repair strategy improves the ability to jump out of the local optimum. Without the repair strategy, the probability of generating $O^{j}$ from the local optimal solution $X^i$ containing $v_1$ is at least $(1/|P|) \cdot(1/n)^3(1-1/n)^{n-3} \ge 1/(en^3(k+1))$, where $(1/n)^3(1-1/n)^{n-3}$ is the probability of flipping two 0-bits (corresponding to $v_2$ and $v_3$) and one 1-bit (corresponding to $v_1$) of $X^i$ while keeping other bits unchanged. Thus, the expected number of iterations to escape from local optima is $O(kn^3)$, which is only $O(kn^2)$ if using the repair strategy.

\section{Empirical Study}\label{experiments}

In this section, we examine the performance of PVD-EMO on a peptide vaccine design for COVID-19, by comparing its two variants, PVD-GSEMO-WR and PVD-NSGA-II-WR, with Optivax-P, the state-of-the-art greedy algorithm that has outperformed thirty other algorithms in the peptide vaccine design as reported in~\cite{dai2023constrained}. Note that PVD-GSEMO-WR and PVD-NSGA-II-WR correspond to PVD-EMO using GSEMO and NSGA-II, respectively. 

The experiments are mainly to answer two questions: Whether any variant of PVD-EMO is better than the previous algorithm Optivax-P? Can the warm-start and repair strategies improve efficiency and performance, respectively? 

\begin{figure}
    \centering
    \includegraphics[width=0.74\linewidth]{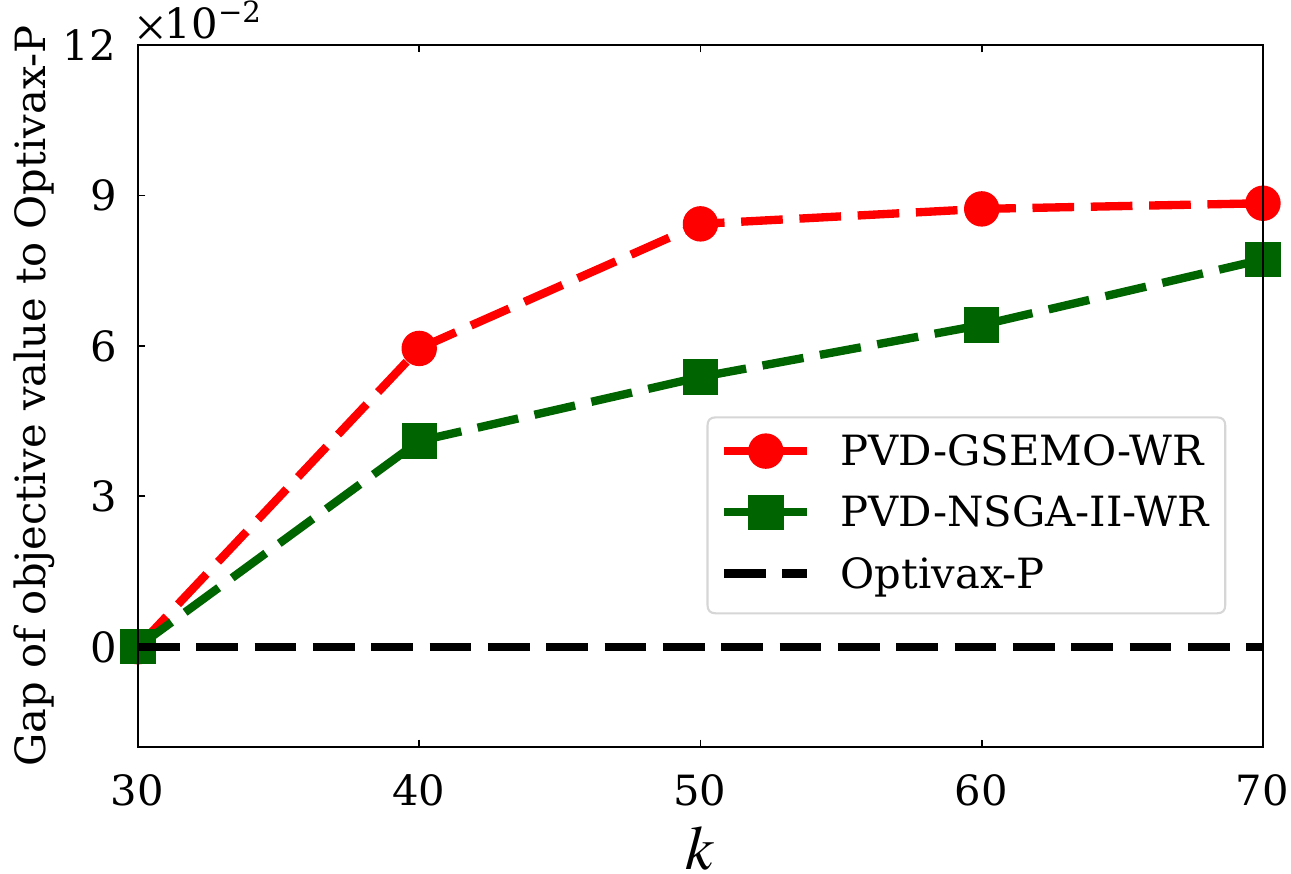}
    \caption{The average objective value of each algorithm minus the objective value of Optivax-P (the larger, the better).}
    \label{Objective-value-MHC-1}
\end{figure}

The initial populations of the two variants of PVD-EMO are generated using the warm-start strategy. Inspired by the recent theoretical work~\cite{zheng2022first}, the population size of PVD-NSGA-II-WR is set to 2 times the size of the Pareto front, i.e., $2(k+1)$. The warm-start strategy creates one feasible solution for each size $0,1,\cdots,k$ for PVD-GSEMO-WR, and two feasible solutions for each size for PVD-NSGA-II-WR, where the output of Optivax-P is used as one of the solutions for size $k$. PVD-NSGA-II-WR applies one-point crossover and bit-wise mutation with probabilities of 0.9 and 1, respectively. As PVD-EMO is an anytime algorithm, whose performance will be gradually improved by increasing the number of iterations, we set the number of objective evaluations to $20kn$, to make a trade-off between the performance and runtime, compared to $kn$ evaluations used by Optivax-P. Also it is randomized, and thus we run it ten times independently and report the average values. 

We use the same dataset of producing a peptide vaccine for COVID-19 as in~\cite{dai2023constrained}, which consists of a set of candidate peptides ($|V|=1043$), a set of genotypes ($|M|=1018459$) for Major Histocompatibility Complex class I (MHC-I), their frequencies $w(m)$ derived from diverse populations to be representative of the global population, and the binding probability $p_{v,m}$ for each peptide-MHC pair, generated by the SOTA neural network-based model NetMHCpan~\cite{NetMHCpan2}. The set of pairwise constraints require that no pair of peptides can be within 6 edits (insertions, deletions, or substitutions) of each other.

We set $k\!\in\!\{30,40,\ldots,70\}$ and $N\!=\!\lfloor 0.25k\rfloor$. The results are shown in Figure~\ref{Objective-value-MHC-1}. It can be observed that the performance of PVD-EMO will be at least as good as the competitive greedy algorithm Optivax-P. For a relatively simple problem ($k\!=\!30$), Optivax-P may find the optimal solution, in which case PVD-EMO performs equally well. However, as the problem complexity increases ($k\in\{40,50,60,70\}$), both PVD-GSEMO-WR and PVD-NSGA-II-WR surpass Optivax-P, showing the superiority of the PVD-EMO framework. This may be because PVD-EMO naturally maintains a population of diverse solutions due to the bi-objective transformation, the employed bit-wise mutation operator has a good global search ability, and the repair strategy further improves the search ability. These characteristics can lead to a better ability of escaping from local optima. Among the two variants of PVD-EMO, PVD-GSEMO-WR performs better than PVD-NSGA-II-WR, which may be because the population of NSGA-II may contain redundant dominated solutions, leading to the bad performance.

\begin{figure}
    \centering
    \includegraphics[width=0.76\linewidth]{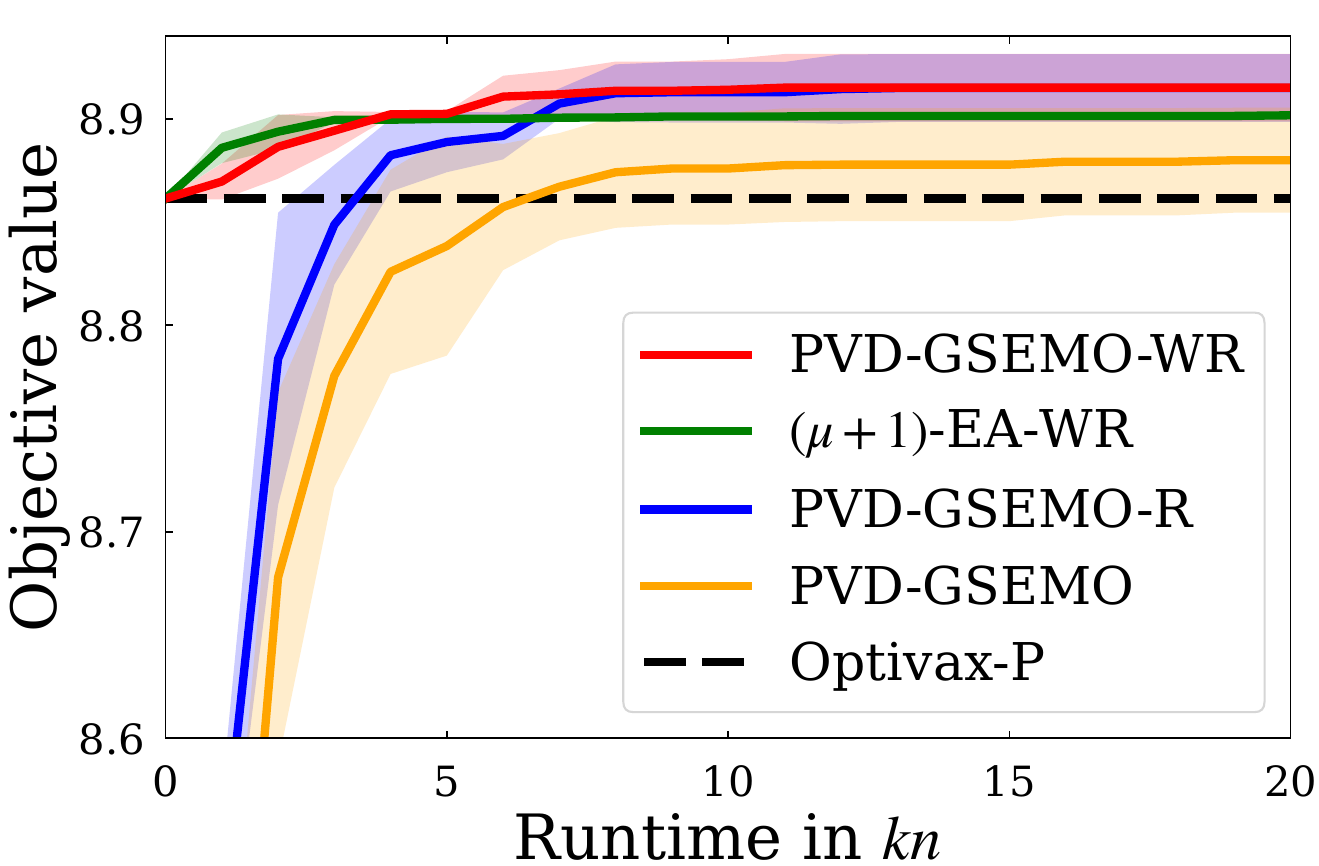}
    \caption{The average objective value $\pm$ the standard deviation vs.~runtime (i.e., number of objective evaluations) with $k=40$.}
    \label{Runtime-MHC-1}
\end{figure}

To more clearly examine the effectiveness of bi-objective reformulation, we conduct experiments using the single-objective EA $(\mu+1)$-EA. For fair comparison, $(\mu+1)$-EA, with a population size of $\mu=k+1$, follows the settings of PVD-GSEMO-WR and is labeled as $(\mu+1)$-EA-WR. The key difference is that $(\mu+1)$-EA-WR updates its population by retaining the best $\mu$ solutions according to their objective $f$ values, whereas PVD-GSEMO-WR preserves non-dominated solutions due to the bi-objective reformulation. Additionally, to assess the impact of warm-start and repair strategies, we test two variants: PVD-GSEMO-R, which solely employs the repair strategy, and PVD-GSEMO with no strategies. We plot the curve of the objective value over runtime with $k=40$ in Figure~\ref{Runtime-MHC-1}. Optivax-P is a fixed-time (nearly $kn$) algorithm, while the others are anytime algorithms that can achieve better performance with increased runtime. The results show that even without warm-start, both PVD-GSEMO-R and PVD-GSEMO outperform Optivax-P within $4kn$ and $7kn$, respectively. PVD-GSEMO-WR and PVD-GSEMO-R converge to the best objective value, while PVD-GSEMO-WR is faster due to the warm-start advantage. In comparison, $(\mu+1)$-EA-WR, despite employing both strategies, attains the second-best objective value, thus validating the value of bi-objective reformulation to maintain a diverse population. Meanwhile, PVD-GSEMO, lacking the repair strategy, only achieves the third-best objective value, underscoring the vital role of the repair strategy in avoiding local optima, as demonstrated in Theorem~\ref{theo-example}. These findings validate the effectiveness of bi-objective reformulation along with the warm-start and repair strategies. 

\section{Conclusion}
This paper proposes the PVD-EMO framework, employing any MOEA to solve the bi-objective reformulated PVD problem. We incorporate warm-start and repair strategies to improve efficiency and performance. We prove that the warm-start strategy ensures that PVD-EMO maintains the same worst-case approximation guarantee as the state-of-the-art greedy algorithm Optivax-P. Furthermore, we prove that unlike Optivax-P, PVD-EMO can successfully avoid getting stuck in local optimal solutions. Empirical results on a peptide vaccine design for COVID-19 show that PVD-EMO can achieve better performance than Optivax-P. An interesting future work is to design better MOEAs (e.g., using fast mutation operator~\cite{fastmutation}) for PVD-EMO.
\section*{Acknowledgments}
The authors want to thank the anonymous reviewers for their helpful comments and suggestions. This work was supported by the National Science and Technology Major Project (2022ZD0116600) and National Science Foundation of China (62276124). Chao Qian is the corresponding author.

\bibliographystyle{named}
\bibliography{ijcai24}

\end{document}